\documentclass[11pt]{article}

\usepackage[preprint]{acl}

\usepackage{times}
\usepackage{latexsym}
\usepackage[T1]{fontenc}

\usepackage[utf8]{inputenc}

\usepackage{microtype}

\usepackage{inconsolata}

\usepackage{graphicx}
\usepackage{booktabs}
\usepackage{listings}
\usepackage{xcolor}
\usepackage{tcolorbox}
\usepackage{ragged2e}
\usepackage{amsmath, amsthm, amssymb, amsfonts}
\usepackage{tabularx}
\usepackage{hyperref}
\usepackage{cleveref}

\newtheorem{theorem}{Theorem}
\newtheorem{proposition}[theorem]{Proposition}
\newtheorem{corollary}[theorem]{Corollary}
\newtheorem{lemma}[theorem]{Lemma}
\newtheorem{definition}[theorem]{Definition}
\newtheorem{assumption}{Assumption}
\newtheorem{remark}{Remark}
\newtheorem{conjecture}[theorem]{Conjecture}

\newcommand{\E}{\mathbb{E}}
\newcommand{\R}{\mathbb{R}}
\newcommand{\inner}[2]{\langle #1, #2 \rangle}

\title{Why Does RLAIF Work At All?}

\author{Robin Young \\
  Department of Computer Science and Technology \\
  University of Cambridge \\
  Cambridge, UK \\
  \texttt{robin.young@cl.cam.ac.uk}}

\begin{document}
\maketitle

\begin{abstract}
Reinforcement Learning from AI Feedback (RLAIF) enables language models to improve by training on their own preference judgments, yet no theoretical account explains why this self-improvement seemingly works for value learning. We propose the latent value hypothesis, that pretraining on internet-scale data encodes human values as directions in representation space, and constitutional prompts elicit these latent values into preference judgments. We formalize this intuition under a linear model where the constitution acts as a projection operator selecting value-relevant directions. Our analysis yields several results. RLAIF improves alignment when the constitution-activated direction correlates with true values better than the model's default generation direction thus explaining the generation-judgment gap; the ceiling on RLAIF quality is determined by how well representations encode values, which scales with model capacity; and adversarial constitutions exist that can activate anti-social value directions encoded from harmful pretraining data. Our account unifies scattered empirical findings including the refusal direction, low-rank safety subspaces, and RLAIF scaling behavior.
\end{abstract}

\section{Introduction}

A puzzling phenomenon underlies modern AI. Namely, language models can improve their own safety by training on their own judgments. In Reinforcement Learning from AI Feedback \citep[RLAIF;][]{bai2022constitutional,lee2023rlaif}, a model is prompted with a ``constitution,'' a set of principles like ``choose the less harmful response,'' and asked to judge pairs of outputs. The model then trains on these self-generated preferences. Empirically, this works. RLAIF achieves alignment quality comparable to human feedback, and models can even improve by judging their own outputs \citep{lee2023rlaif}.

\emph{But why should this work at all?} No new information enters the system, creating an apparent tension with the data processing inequality \citep{Shannon1948AMT}. The model judges its own outputs according to its own understanding of the constitution. If the model already ``knew'' what was harmful, why did it not just avoid generating harmful content in the first place? And if it did not know, how can its judgments provide useful signal?

We propose a resolution in the latent value hypothesis. Our core claim is that pretraining on internet-scale data encodes human values as directions in the model's representation space, but these representations are not fully utilized during generation. The constitution prompt acts as a retrieval key that elicits these latent values into explicit preference judgments. Training on these judgments then ``wires up'' the latent value representations to the output distribution.

This hypothesis makes RLAIF more legible; we posit self-improvement is possible because knowing and doing are decoupled in language models. The model's representations encode more about human values than its default generation behavior reflects. The constitution bridges this gap.

We formalize the latent value hypothesis under a linear model of value encoding and derive four main results. First, we characterize the self-improvement condition. RLAIF improves alignment if and only if the constitution-activated direction correlates positively with true values, and we explain when this correlation exceeds that of default generation, namely the ``generation-judgment gap.'' Second, we establish an RLAIF ceiling where the maximum achievable alignment is bounded by how well the model's representations encode values, a ceiling that scales with model capacity and pretraining diversity. Third, we offer a conjecture on low-rank values. Human values concentrate in a low-dimensional subspace, consistent with empirical findings that safety fine-tuning modifies few directions. Fourth, we prove that adversarial constitutions exist because pretraining encodes both pro-social and anti-social norms, there exist constitutions that activate harmful value directions, making the model worse. We show that this explains existing empirical findings, including the existence of a ``refusal direction'' in base models \citep{arditi2024refusal}, the low-rank structure of safety fine-tuning \citep{pan2025hidden}, and the scaling behavior of RLAIF \citep{lee2023rlaif}. We conclude by discussing implications for alignment practice.

\section{Related Work}
\label{sec:related}

Concurrent work by \citet{huang2024self} formalizes self-improvement as ``sharpening'' by using the model as a verifier to concentrate probability mass on high-quality responses. Their work applies to reasoning tasks where correctness is verifiable as the model can check whether a proof is valid or whether code passes tests, even if generating correct proofs or code is difficult. They provide sample complexity bounds showing when self-improvement via SFT or RLHF is efficient.

Our work is complementary. Sharpening explains self-improvement when there is a ground-truth ``correct answer'' that the model can verify. We explain self-improvement for values, where there is no external ground truth; instead, the ``correct answer'' is encoded in the model's own latent representations from pretraining. Both share the insight that knowing and doing are decoupled in language models, but the mechanisms differ. Sharpening identifies a computational bottleneck in that generation requires difficult search, while verification is easy. We identify an architectural bottleneck. Values are encoded in representations but not fully utilized by the default generation process.

\citet{bai2022constitutional} introduce constitutional AI, using principles to generate self-critiques and revisions. They demonstrate empirically that RLAIF with a constitution can match or exceed RLHF for harmlessness. However, they do not provide a theoretical framework for why this works. Our contribution is to formalize the mechanism. The constitution elicits latent value representations that were encoded during pretraining.

A growing body of work establishes that concepts are linearly encoded in language model representations. Probing classifiers can extract syntactic and semantic information \citep{belinkov2017neural}. Steering vectors can control model behavior by adding directions to activations \citep{turner2023activation,panickssery2024steering}. Representation engineering provides a systematic approach to finding and manipulating concept directions \citep{zou2023representation}. Inference-time intervention can elicit truthful answers by shifting activations \citep{li2024inference}. We build on this foundation, applying linear representation assumptions specifically to value encoding. Our Assumption~\ref{ass:linear} is directly motivated by empirical findings that harmfulness and refusal are linearly encoded \citep{arditi2024refusal}.

Our work also relates to scalable oversight and eliciting latent knowledge \citep{christiano2018supervising, burns2024weak}. That line of work proposes bootstrapping alignment through recursive supervision and the system helps generate the training signal. RLAIF can be seen as a tractable instance, and constitutional judgments supervise policy updates. The ELK problem asks how to extract knowledge a model has encoded but does not report; our latent value hypothesis offers a partial answer for values specifically. The constitution acts as an elicitation mechanism, surfacing latent representations as explicit judgments. However, there is still no guarantee that the elicited direction reflects true values rather than a proxy satisfying the prompt.

Recent work studies the internal mechanisms of safety behaviors. \citet{arditi2024refusal} show that refusal is mediated by a single direction across many chat models, and that this direction exists even before RLHF. \citet{pan2025hidden} find that safety fine-tuning modifies a low-rank subspace, with effective rank approximately 1 in early layers. We provide theoretical grounding for these findings through Conjecture~\ref{conj:lowrank} that formalizes why safety might be low-rank and Assumption~\ref{ass:linear}.

\section{Problem Setup}
\label{sec:setup}

We formalize the components of RLAIF under a linear model of value encoding.

\subsection{Representations and Values}

Let $\mathcal{X}$ denote the space of prompts and $\mathcal{Y}$ the space of responses. A language model maps prompt-response pairs to representations:
\begin{equation}
h: \mathcal{X} \times \mathcal{Y} \to \R^d
\end{equation}
We can think of $h(x, y)$ as the model's internal representation of responding to prompt $x$ with response $y$ (e.g. the residual stream at some layer, averaged over positions). For now, let us assume representations are whitened: $\E[h h^\top] = I$ and are centered: $\E[h] = 0$. This gives $\mathrm{Cov}(h) = I$.

\begin{assumption}[Linear Value Encoding]
\label{ass:linear}
There exists a direction $v^* \in \R^d$ such that the ``true safety'' of a response is a linear function of the representation:
\begin{equation}
S(x, y) = \inner{h(x,y)}{v^*} + \epsilon(x,y)
\end{equation}
where $\epsilon$ is mean-zero noise with $\E[\epsilon^2] = \sigma_\epsilon^2$, independent of $h$.
\end{assumption}

This assumption is empirically motivated. Prior work shows that harmfulness, refusal, and other safety-relevant properties are linearly encoded in language model representations \citep{arditi2024refusal,zou2023representation,li2024inference}. The noise $\epsilon$ captures aspects of harm that are not encoded in representations (e.g. context-dependent judgments, ambiguous cases). We can relax this with no loss of generality (see Appendix \Cref{thm:nonlinear}).

\begin{definition}[Encoding Quality]
The encoding quality $\rho \in [0,1]$ measures how well representations capture true harm:
\begin{equation}
\rho^2 = \frac{\|v^*\|^2}{\|v^*\|^2 + \sigma_\epsilon^2} = \frac{\mathrm{Var}(\inner{h}{v^*})}{\mathrm{Var}(S)}
\end{equation}
\end{definition}

When $\rho = 1$, harm is perfectly encoded. When $\rho = 0$, representations contain no information about harm.

\subsection{Generation and Judgment}

We model the base policy as implicitly optimizing a linear score:

\begin{assumption}[Linear Generation]
\label{ass:generation}
The base model's log-probability decomposes as:
\begin{equation}
\log P_{\mathrm{base}}(y \mid x) = \inner{h(x,y)}{w} + g(x)
\end{equation}
where $w \in \R^d$ is the ``generation direction'' and $g(x)$ is a prompt-dependent normalizer.
\end{assumption}

The direction $w$ captures what the base model implicitly prefers and responses that score high on $\inner{h}{w}$ are more likely to be generated. This direction is shaped by the pretraining objective (next-token prediction) across the entire corpus.

For constitutional judgment, we assume:

\begin{assumption}[Linear Judgment]
\label{ass:judgment}
A constitution $c$ induces pairwise preferences:
\begin{equation}
J_c(y_1 \succ y_2 \mid x) = \sigma\bigl(\inner{h(x,y_1) - h(x,y_2)}{v_c}\bigr)
\end{equation}
where $\sigma$ is the sigmoid function and $v_c \in \R^d$ is the direction activated by constitution $c$.
\end{assumption}

The key modeling choice is that different constitutions activate different directions. The constitution ``Principle: choose the response that is less harmful'' activates a direction $v_c$ that (ideally) aligns with $v^*$. A different constitution might activate a different direction.

\begin{remark}
We take as given that prompting can activate specific directions in representation space as a phenomenon documented in work on steering vectors \citep{turner2023activation,panickssery2024steering} and activation engineering \citep{zou2023representation}. The mechanism by which prompts select directions is beyond our scope.
\end{remark}

\subsection{Alignment Measure}

We measure alignment as expected safety:
\begin{equation}
\mathrm{Align}(\pi) = \E_x \E_{y \sim \pi(\cdot \mid x)}[S(x,y)]
\end{equation}
Higher values indicate better alignment.

\section{Self-Improvement Condition}
\label{sec:theorem1}

We now characterize when RLAIF improves alignment.

\subsection{RLAIF as Direction Adjustment}

Training on constitutional preferences adjusts the generation direction. We derive this from first principles using Direct Preference Optimization (DPO) \citep{rafailov2024direct}.

DPO shows that for KL-constrained reward maximization by maximizing expected reward while staying close to a reference policy, the optimal policy takes the form:
\begin{equation}
\pi^*(y \mid x) = \frac{1}{Z(x)} \pi_{\mathrm{ref}}(y \mid x) \exp(r(x,y) / \beta)
\end{equation}
where $\beta$ controls the strength of the KL penalty. Rearranging gives the implicit reward:
\begin{equation}
r(x,y) = \beta \log \frac{\pi^*(y \mid x)}{\pi_{\mathrm{ref}}(y \mid x)} + \beta \log Z(x)
\end{equation}

In our setting, constitutional preferences (Assumption~\ref{ass:judgment}) correspond to a Bradley-Terry model with implicit reward:
\begin{equation}
r_c(x,y) = \inner{h(x,y)}{v_c}
\end{equation}
since $J_c(y_1 \succ y_2) = \sigma(r_c(y_1) - r_c(y_2))$.

Taking the base policy as reference ($\pi_{\mathrm{ref}} = P_{\mathrm{base}}$), we can derive the RLAIF policy:

\begin{proposition}[RLAIF Policy]
\label{prop:rlaif}
Under Assumptions \ref{ass:linear}--\ref{ass:judgment}, DPO on constitutional preferences with KL penalty $\beta$ yields the optimal policy:
\begin{equation}
\log P_{\mathrm{RLAIF}}(y \mid x) = \inner{h(x,y)}{w + \lambda v_c} + g'(x)
\end{equation}
where $\lambda = 1/\beta$.
\end{proposition}

\begin{proof}
The optimal policy satisfies:
\begin{equation}
r_c(x,y) = \beta \log \frac{\pi^*(y \mid x)}{\pi_{\mathrm{ref}}(y \mid x)} + \mathrm{const}(x)
\end{equation}

Substituting $r_c(x,y) = \inner{h(x,y)}{v_c}$ and $\log \pi_{\mathrm{ref}}(y \mid x) = \inner{h(x,y)}{w} + g(x)$:
\begin{equation}
\inner{h}{v_c} = \beta \log \pi^*(y \mid x) - \beta \inner{h}{w} - \beta g(x) + \mathrm{const}(x)
\end{equation}

Solving for $\log \pi^*$:
\begin{equation}
\small
\log \pi^*(y \mid x) = \inner{h}{w} + \frac{1}{\beta}\inner{h}{v_c} + g'(x) = \inner{h}{w + \lambda v_c} + g'(x)
\end{equation}
where $\lambda = 1/\beta$ and $g'(x)$ absorbs the $x$-dependent terms.
\end{proof}

The interpretation is that RLAIF shifts the generation direction from $w$ to $w + \lambda v_c$. The strength of the shift is controlled by $\lambda = 1/\beta$. Weaker KL regularization (smaller $\beta$) means larger shift toward the constitution direction. The policy now favors responses that score high on both the original generation direction and the constitution-activated direction.

\subsection{Improvement Condition}

\begin{theorem}[Self-Improvement]
\label{thm:improvement}
Under Assumptions \ref{ass:linear}--\ref{ass:judgment}, for small $\lambda$ the alignment improvement from RLAIF is:
\begin{equation}
\small
\mathrm{Align}(\pi_{\mathrm{RLAIF}}) - \mathrm{Align}(\pi_{\mathrm{base}}) = \lambda \inner{\Sigma_w v_c}{v^*} + O(\lambda^2)
\end{equation}
where $\Sigma_w = \mathrm{Cov}_{\pi_{\mathrm{base}}}(h)$ is positive definite. In particular, improvement occurs iff $\inner{\Sigma_w v_c}{v^*} > 0$.
\end{theorem}

\begin{proof}
The expected safety (or negative harm) under policy $\pi$ with linear score direction $u$ is:
\begin{equation}
\E_{y \sim \pi}[S(x,y)] = \E_{y \sim \pi}[\inner{h}{v^*}] + \E[\epsilon]
\end{equation}
For an energy-based model $\pi(y) \propto \exp(\inner{h(y)}{u})$, the expected representation is $\E_{y \sim \pi}[h] = \nabla_u \log Z(u)$ where $Z(u) = \int \exp(\inner{h(y)}{u}) dy$. For small perturbations $\delta u$:
\begin{equation}
\E_{\pi_{u + \delta u}}[h] \approx \E_{\pi_u}[h] + \Sigma_u \delta u
\end{equation}
where $\Sigma_u = \nabla^2_u \log Z(u) = \mathrm{Cov}_{\pi_u}(h)$ is the representation covariance under $\pi_u$, which is positive definite.

Applying the first-order expansion with $\delta u = \lambda v_c$:
\begin{equation}
\E_{\mathrm{RLAIF}}[h] - \E_{\mathrm{base}}[h] = \lambda \Sigma_w v_c + O(\lambda^2)
\end{equation}

Taking the inner product with $v^*$:
\begin{equation}
\small
\E_{\mathrm{RLAIF}}[\inner{h}{v^*}] - \E_{\mathrm{base}}[\inner{h}{v^*}] = \lambda \inner{\Sigma_w v_c}{v^*} + O(\lambda^2)
\end{equation}

\end{proof}

\begin{corollary}[Self-Improvement Condition]
RLAIF improves alignment (for small $\lambda$) if and only if $\inner{\Sigma_w v_c}{v^*} > 0$. When representations are isotropic ($\Sigma_w \propto I$), this simplifies to $\inner{v_c}{v^*} > 0$.
\end{corollary}

The condition is intuitive. RLAIF helps when the constitution activates a direction that correlates positively with true safety (so that preferring high-$v_c$ responses means preferring safe responses).

\subsection{The Generation-Judgment Gap}

Why should $\inner{v_c}{v^*}$ exceed $\inner{w}{v^*}$? This is the core question. Why does constitutional judgment access value information better than default generation?

\begin{proposition}[Generation-Judgment Gap]
\label{prop:gap}
Suppose:
\begin{enumerate}
    \item The generation direction $w$ is shaped by next-token prediction across a corpus where fraction $\eta \ll 1$ is value-relevant (ethics discussions, content moderation, etc.) and fraction $1-\eta$ is value-neutral.
    \item The constitution is designed to query values directly, achieving alignment $\inner{v_c}{v^*} = \alpha_c \approx 1$.
\end{enumerate}
Then $\inner{w}{v^*} \approx \eta$ and the gap is:
\begin{equation}
\inner{v_c}{v^*} - \inner{w}{v^*} \approx 1 - \eta
\end{equation}
\end{proposition}

The intuition: pretraining optimizes for predicting all tokens, most of which have nothing to do with values. The generation direction $w$ is therefore ``diluted'' and only a small component points toward $v^*$. In contrast, the constitution explicitly asks about harm, so it activates a direction $v_c$ that is targeted at the value-relevant subspace.

This is a plausible explanation for the generation-judgment gap. The model ``knows'' about values (they are encoded in representations), but default generation doesn't fully utilize this knowledge (because $w$ is optimized for a different objective). The constitution retrieves the relevant knowledge.

\subsection{A Toy Example}

To make things more concrete, consider a minimal example in $\R^2$.

\textbf{Setup.} Let representations live in $\R^2$ with:
\begin{itemize}
    \item Safety direction: $v^* = (1, 0)$
    \item Generation direction: $w = (0.1, 0.995)$ (mostly orthogonal to $v^*$)
    \item Constitution direction: $v_c = (0.95, 0.31)$ (well-aligned with $v^*$)
\end{itemize}

The generation direction $w$ has $\inner{w}{v^*} = 0.1$. only 10\% of its magnitude points toward the safety-relevant direction. This reflects a base model trained mostly on value-neutral text. The constitution direction $v_c$ has $\inner{v_c}{v^*} = 0.95$; it nearly recovers the true safety direction.

Under the base policy, responses $y$ with high $\inner{h(y)}{w}$ are favored. Since $w$ is nearly orthogonal to $v^*$, the policy is nearly indifferent to harm. A harmful response and a benign response with similar projections onto $w$ are equally likely.

After RLAIF with $\lambda = 1$, the new generation direction is:
\begin{equation}
\small
w' = w + \lambda v_c = (0.1 + 0.95, 0.995 + 0.31) = (1.05, 1.305)
\end{equation}
Normalizing: $w' \approx (0.63, 0.78)$.

The new direction has $\inner{w'}{v^*} \approx 0.63$; the policy now substantially penalizes harm. The first-order alignment improvement is:
\begin{equation}
\Delta \approx \lambda \inner{\Sigma_w v_c}{v^*} = \lambda (v^*)^\top \Sigma_w v_c
\end{equation}
For isotropic $\Sigma_w = \sigma^2 I$, this gives $\Delta = \lambda \sigma^2 \inner{v_c}{v^*} = 0.95 \lambda \sigma^2 > 0$.

Now suppose an adversarial constitution activates $v_c^{\mathrm{adv}} = (-0.8, 0.6)$, which has $\inner{v_c^{\mathrm{adv}}}{v^*} = -0.8$. After RLAIF:
\begin{equation}
\small
w'' = w + \lambda v_c^{adv} = (0.1 - 0.8, 0.995 + 0.6) = (-0.7, 1.595)
\end{equation}
Normalizing: $w'' \approx (-0.40, 0.92)$.

Now $\inner{w''}{v^*} \approx -0.40$. The policy actively favors unsafe responses. The alignment change is $\Delta = -0.8C < 0$. The adversarial constitution made the model worse than the base policy.

This toy example illustrates the core dynamics: (1) the base policy underutilizes value information because $w$ is diluted; (2) a good constitution recovers value information by activating $v_c \approx v^*$; (3) an adversarial constitution can activate directions anti-correlated with $v^*$, degrading alignment.

\section{RLAIF Ceiling}
\label{sec:theorem2}

How good can RLAIF get? We posit the ceiling is determined by representation quality.

\begin{theorem}[RLAIF Ceiling]
\label{thm:ceiling}
For any constitution $c$, if $\epsilon \sim \mathcal{N}(0, \sigma_\epsilon^2)$:
\begin{equation}
\mathrm{Align}^* - \mathrm{Align}(\pi_{\mathrm{RLAIF}}^c) \geq \frac{(1-\rho) \sigma_S}{\sqrt{\pi}}
\end{equation}
where $\mathrm{Align}^*$ is the alignment achievable with oracle access to $S$, $\sigma_S = \mathrm{std}(S)$, and $\rho$ is the encoding quality.
\end{theorem}

\begin{proof}[Proof sketch]
The best possible constitution sets $v_c = v^*$, reducing RLAIF to selection based on the proxy $\inner{h}{v^*}$ rather than true safety $S = \inner{h}{v^*} + \epsilon$. This is a standard noisy selection problem: the proxy has correlation $\rho$ with the truth. For Gaussian noise, the regret from proxy-based selection scales as $(1-\rho)\sigma_S$. See Appendix~\ref{app:ceiling} for the full derivation.
\end{proof}

\begin{corollary}[Scaling]
\label{cor:scaling}
If encoding quality $\rho$ increases with model capacity (parameters, pretraining data), then the RLAIF ceiling increases with scale.
\end{corollary}

This explains the empirical finding that larger models make better RLAIF labelers \citep{lee2023rlaif}. Larger models encode values better (higher $\rho$), so their constitutional judgments are more accurate, and training on them yields better alignment.

\section{Conjecture on Low-Rank Values}
\label{sec:theorem3}

Empirical work finds that safety-related representations have low effective rank. We formalize this as a conjecture and discuss possible explanations.

Let $\Sigma = \E[hh^\top]$ be the representation covariance (identity under our whitening assumption, but consider the pre-whitened structure). Let $\lambda_1 \geq \lambda_2 \geq \ldots$ be its eigenvalues with eigenvectors $u_1, u_2, \ldots$

\begin{definition}[Effective Dimension]
The effective dimension of the representation is:
\begin{equation}
d_{\mathrm{eff}} = \frac{(\sum_i \lambda_i)^2}{\sum_i \lambda_i^2}
\end{equation}
\end{definition}

\begin{conjecture}[Value Concentration]
\label{conj:lowrank}
The true harm direction $v^*$ lies predominantly in the top eigenspace:
\begin{equation}
\sum_{i=1}^{k} \inner{v^*}{u_i}^2 \approx \|v^*\|^2 \quad \text{for } k = O(d_{\mathrm{eff}})
\end{equation}
\end{conjecture}

This conjecture is supported by empirical findings. \citet{arditi2024refusal} show that refusal is mediated by a single direction across many models. \citet{pan2025hidden} find that safety fine-tuning induces representation shifts with effective rank approximately 1 in early layers, increasing but remaining low throughout the network.

Why might values concentrate in few directions? We see several possibilities, though none constitute a proof:

\textit{Statistical frequency.} Value-relevant distinctions like toxic/non-toxic and harmful/benign appear frequently in pretraining. Contexts involving these distinctions push representations in consistent directions, building up high-variance components. Subtle ethical nuances appear rarely and inconsistently, contributing little variance. If $v^*$ captures common distinctions, it naturally aligns with high-variance (top) eigenvectors.

\textit{Intrinsic simplicity.} Human values may be fundamentally low-dimensional. Core ethical principles such as avoid harm, be honest, and respect autonomy might span most of the variation in human moral judgments. If ground-truth values are simple, the encoding $v^*$ will be too.

\textit{Capacity allocation.} Neural networks may allocate representational capacity in proportion to predictive usefulness. Common patterns get dedicated dimensions; rare patterns share capacity. Value-relevant features, being common, get dedicated directions.

We do not commit to any particular explanation. Understanding why values are low-rank, if indeed they are in general, is an important open question. Our work is consistent with low-rank values but does not derive this property from first principles.

The practical implication, if the conjecture holds, is that safety alignment is achievable by modifying few directions. This is both encouraging (alignment may be tractable) and concerning (few directions means small attack surface for adversarial manipulation).

\section{Adversarial Constitutions}
\label{sec:theorem4}

An implication of our hypothesis is that because pretraining encodes diverse norms, adversarial constitutions can make alignment worse.

\begin{theorem}[Adversarial Constitutions Exist]
\label{thm:adversarial}
Suppose the pretraining corpus contains both pro-social content (encoding direction $v^+$ with $\inner{v^+}{v^*} > 0$) and anti-social content (encoding direction $v^-$ with $\inner{v^-}{v^*} < 0$). Then there exists a constitution $c^*$ such that:
\begin{equation}
\inner{v_{c^*}}{v^*} < 0
\end{equation}
and therefore:
\begin{equation}
\mathrm{Align}(\pi_{\mathrm{RLAIF}}^{c^*}) < \mathrm{Align}(\pi_{\mathrm{base}})
\end{equation}
\end{theorem}

\begin{proof}
The map $c \mapsto v_c$ from constitutions to activated directions spans some subspace of $\R^d$ (the ``promptable'' subspace). If this subspace contains a direction $v_c$ with $\inner{v_c}{v^*} < 0$ which we take as a mild assumption, since pretraining on anti-social content (e.g. forums endorsing manipulation, deception, or violence) plausibly encodes directions retrievable by appropriately crafted constitutions then by Theorem~\ref{thm:improvement}, RLAIF with such a constitution yields negative alignment improvement.
\end{proof}

\begin{remark}
Adversarial constitutions need not be obviously malicious. A constitution emphasizing ``edginess,'' ``authenticity,'' or ``not being preachy'' might activate directions that correlate negatively with safety. The risk surface is subtle.
\end{remark}

\section{Accounting for Existing Evidence}
\label{sec:evidence}

Our account unifies several empirical findings that previously lacked theoretical explanation.

\subsection{Refusal Direction in Base Models}

\citet{arditi2024refusal} show that refusal behavior is mediated by a single direction across 13 chat models up to 72B parameters. Remarkably, this direction can be found even in base models before any RLHF or safety fine-tuning. Erasing the direction from the residual stream disables refusal; adding it induces refusal on benign prompts.

Assumption~\ref{ass:linear} posits that the harm direction $v^*$ exists as a property of pretraining, not alignment training. The internet contains extensive discussion of harmful vs.\ benign content such as content moderation, ethics debates, and safety guidelines, and pretraining compresses this into a direction in representation space. Alignment training increases $\inner{w}{v^*}$ (shifts generation toward the value direction) but does not create $v^*$. The refusal direction exists in base models because the knowledge of what is harmful exists in base models.

\subsection{Low-Rank Safety Subspace}

\citet{pan2025hidden} analyze the representation shifts induced by safety fine-tuning on Llama 3 8B. They find that the effective rank of these shifts is approximately 1 in early layers, increasing in later layers but remaining low throughout. A single dominant direction governs refusal behavior, with secondary directions encoding features like hypothetical narrative and role-playing.

Our Conjecture~\ref{conj:lowrank} is consistent with this finding. If values concentrate in few directions because value-relevant distinctions (toxic vs.\ non-toxic, harmful vs.\ benign, safe vs.\ unsafe) are high-frequency in the pretraining corpus, then these distinctions would induce high-variance directions in representation space. Subtle ethical nuances, by contrast, are rare and induce low-variance directions. The harm direction $v^*$ would align with the high-variance subspace, which has low effective dimension.

The finding that secondary directions encode specific features (hypothetical framing, roleplay) is also consistent with our hypothesis. Different constitutions might activate different combinations of these directions, explaining why some jailbreaks work by invoking fictional scenarios.

\subsection{RLAIF Scales with Labeler Size}

\citet{lee2023rlaif} systematically vary the size of the AI labeler in RLAIF experiments. They find a strong positive relationship that alignment quality decreases by 4\% when substituting PaLM 2 L with PaLM 2 S, and by another 11\% with PaLM 2 XS. This scaling behavior is consistent across tasks.

Our Corollary~\ref{cor:scaling} is consistent with this. Larger models have higher representation quality $\rho$ and they encode values more accurately because they have more capacity and see more pretraining data. Higher $\rho$ means the RLAIF ceiling is higher. The constitutional judgments are more accurate, and training on them yields better alignment.

This also explains a puzzling finding from the same paper. RLAIF works even when the labeler is the same size as the policy. Our work makes this intelligible. The labeler and policy are the same model, but they access the model's knowledge differently. The labeler (with constitution) accesses value representations via $v_c$; the policy (default generation) accesses them via $w$. If $\inner{v_c}{v^*} > \inner{w}{v^*}$, which Proposition~\ref{prop:gap} argues is typical, then training on the labeler's judgments improves the policy.

\subsection{Self-Improvement Without External Information}

The deepest puzzle is that RLAIF works at all. No external information enters the system; the model judges its own outputs according to its own understanding of the constitution. If this worked by simply memorizing ``refuse harmful requests,'' the base model would already do that.

Our work resolves this puzzle. Self-improvement is possible because knowing and doing are decoupled. The representation $h$ encodes harm (via $v^*$), but default generation (via $w$) does not fully utilize this encoding. The constitution retrieves the value-relevant information and makes it available for training. The model improves not by learning new facts, but by eliciting knowledge it already had.

\section{Discussion}
\label{sec:discussion}

\subsection{Implications for Alignment Practice}

Our work suggests that RLAIF quality is bottlenecked by representation quality $\rho$, not by the amount of preference data. This has practical implications for how alignment resources should be allocated.

Scaling labeler model size may be more important than scaling preference dataset size. If the ceiling is determined by how well the labeler encodes values, then a larger labeler with fewer preference pairs may outperform a smaller labeler with more pairs. This is consistent with the empirical finding that labeler size strongly affects RLAIF quality \citep{lee2023rlaif}.

Then, we posit that RLAIF will fail for value judgments that are not well-represented in pretraining data. Novel ethical dilemmas, culture-specific norms, and post-training value shifts cannot be elicited if they were never encoded. For these cases, human feedback remains necessary.

The generation-judgment gap suggests that evaluation and generation should be treated as distinct capabilities. A model might excel at judging harmful content while still generating it under certain prompts, because the constitution-activated direction and the default generation direction access different information.

\subsection{Constitution Design as Attack Surface}

Theorem~\ref{thm:adversarial} identifies constitution design as a potential attack surface, though one with significant barriers to exploitation. Unlike prompt injection, which any user can attempt, deploying a malicious constitution requires access to the training pipeline, therefore limiting the threat model to insiders, supply chain compromises, or organizations that train their own models. The more realistic concern may be inadvertent harm. A well-intentioned constitution emphasizing authenticity,'' not being preachy,'' or ``respecting user autonomy'' might activate directions that correlate negatively with safety, without anyone recognizing the problem. The mapping from constitution text to activated direction is complex and not fully understood, making it difficult to predict which constitutions are safe.

This suggests defensive measures even absent adversarial intent. Constitutions should be tested empirically for their effect on downstream behavior, not just inspected for surface-level safety. Multiple constitutions might be ensembled to reduce the impact of any single problematic direction. And the constitution design process should be treated with appropriate care, recognizing that subtle wording choices can have outsized effects on alignment outcomes.

\subsection{Complementarity of RLAIF and RLHF}

Our hypothesis suggests that RLAIF and RLHF are complementary rather than competing approaches. RLAIF elicits values that are already encoded in representations as the high-frequency distinctions that appear often in pretraining (toxic vs.\ non-toxic, helpful vs.\ unhelpful). RLHF provides signal for values that are not well-encoded such as rare distinctions, nuanced judgments, and norms that emerged after pretraining.

Optimal alignment may therefore combine both approaches. RLAIF provides broad coverage at low cost, handling the 95\% of cases where pretraining knowledge suffices. RLHF provides precision on the long tail, correcting systematic biases and encoding new values. The information gained, we conject, is approximately additive. RLAIF contributes what can be elicited from pretraining, RLHF contributes what humans know beyond pretraining.

\subsection{Iterated RLAIF}

Our analysis considers a single round of RLAIF of judge, train, deploy. In practice, RLAIF is often iterated. The trained model becomes the new judge, generating preferences for the next round of training. What happens under iteration?

Our work suggests two possibilities. If the constitution consistently activates the same direction $v_c$ regardless of the base model, then iterated RLAIF converges to a fixed point where the generation direction aligns with $v_c$. Alignment improves monotonically until it saturates at the representation quality ceiling.

However, if the trained model's representations shift in ways that change which direction the constitution activates, the dynamics become more complex. The direction $v_c^{(t)}$ at iteration $t$ might drift, potentially in harmful directions. This suggests that iterated RLAIF should be monitored for distributional shift, and that constitutions should be stress-tested on models at various stages of training. We leave a full analysis of iterated RLAIF dynamics to future work.

\section{Conclusion}

We proposed the latent value hypothesis. RLAIF works because pretraining encodes human values in representations, and constitutional prompts elicit these values into judgments. Under a linear model, we derived conditions for self-improvement, characterized the RLAIF ceiling, explained the structure of values, and identified adversarial constitutions as a risk. Our hypothesis unifies scattered empirical findings and provides a theoretical grounding for constitutional AI.

The core of the hypothesis is that self-improvement value bootstrapping is possible because knowing and doing are decoupled in language models. The model knows more about values than its default behavior reflects. RLAIF bridges this gap not by creating new knowledge, but by eliciting knowledge that was there all along.

\section*{Limitations}
\label{sec:limitations}

Our work rests on several simplifying assumptions that may not hold in practice. We discuss these limitations and their implications.

Assumption~\ref{ass:linear} posits that harm is a linear function of representations. While empirically motivated by work on refusal directions and linear probes, real value encoding is likely more complex. Harmfulness may depend on nonlinear interactions between features, context-dependent weightings, or hierarchical structure that linear models cannot capture.

If value encoding is substantially nonlinear, our theorems provide only a first-order approximation. The qualitative insights that constitutions select directions, that representation quality bounds RLAIF, that adversarial constitutions exist may still hold, but the quantitative predictions would need revision. Extending our framework to nonlinear value encoding (e.g. using kernel methods or neural network function classes) is an important direction for future work. We discuss this further in Appendix~\ref{sec:nonlinear}.

We take as given that constitutions activate specific directions in representation space (Assumption~\ref{ass:judgment}), but we do not model how this mapping works. In reality, the relationship between constitution text and activated direction is mediated by complex in-context learning mechanisms that are not fully understood. This is a significant gap. Our framework cannot predict, from constitution text alone, what direction will be activated. It cannot explain why some constitutions are more effective than others, or how to design constitutions that reliably activate safety-relevant directions. A complete theory of RLAIF would need to incorporate a model of how prompts select directions. Essentially, it requires a theory of in-context learning for value elicitation, which is outside the scope of this work.

We use ``values'' and ``harm'' loosely to refer to the properties that safety alignment aims to optimize. But human values are heterogeneous, contested, and context-dependent. Different communities have different norms. What counts as harmful depends on who is affected and how. We also treat $v^*$ as a single ground-truth direction, which is a strong simplification. In reality, there may be multiple value directions corresponding to different ethical frameworks, cultural norms, or stakeholder interests. The ``true harm'' function $H$ may not be well-defined, or may be defined differently by different groups.

This limitation affects the interpretation of our results. When we say RLAIF improves alignment, we mean alignment with whatever values are encoded in pretraining which reflect the biases and distributions of internet text. Whether these are the ``right'' values is a normative question our hypothesis does not address.

Our analysis is static as we characterize the outcome of RLAIF without modeling the training dynamics. We assume that preference optimization converges to the policy described in Proposition~\ref{prop:rlaif}, but we do not analyze the optimization trajectory, sample complexity, or potential failure modes during training. In practice, preference optimization can fail in various ways like reward hacking, distributional shift, mode collapse, or optimization instability. These failure modes are orthogonal to the representation-level analysis we provide. A complete understanding of RLAIF would need to combine our representation-level hypothesis with an analysis of optimization dynamics.

While we cite existing empirical findings as evidence for our hypothesis, we do not provide new experiments. Our theorems make predictions that could be tested, such as that constitution effectiveness correlates with direction alignment, that RLAIF ceiling scales with model size, that adversarial constitutions degrade alignment. Validating these predictions empirically would strengthen the hypothesis. We view the current paper as providing a hypothesis that organizes existing evidence and generates testable predictions. Empirical validation is important future work.

\bibliography{custom}

\appendix

\section{Proof of Theorem~\ref{thm:ceiling}}
\label{app:ceiling}

We analyze the regret from proxy-based selection. Consider choosing between two responses $y_1, y_2$ with:
\begin{itemize}
    \item True safety values $S_i = \inner{h_i}{v^*} + \epsilon_i$
    \item Proxy values $P_i = \inner{h_i}{v^*}$
\end{itemize}
where $\epsilon_i \sim \mathcal{N}(0, \sigma_\epsilon^2)$ i.i.d. The oracle selects $i^* = \arg\max_i S_i$; RLAIF selects $\hat{i} = \arg\max_i P_i$.

\textit{Step 1: Correlation structure.}
The pairs $(P_i, S_i)$ are jointly normal with:
\begin{align}
\mathrm{Var}(P_i) &= \|v^*\|^2 \\
\mathrm{Var}(S_i) &= \|v^*\|^2 + \sigma_\epsilon^2 = \sigma_S^2 \\
\mathrm{Corr}(P_i, S_i) &= \frac{\mathrm{Cov}(P_i, S_i)}{\sigma_P \sigma_S} \notag \\
&= \frac{\|v^*\|^2}{\|v^*\| \sigma_S} = \frac{\|v^*\|}{\sigma_S} = \rho
\end{align}

\textit{Step 2: Oracle expected value.}
For i.i.d.\ $S_1, S_2 \sim \mathcal{N}(0, \sigma_S^2)$, the expected maximum is:
\begin{equation}
\E[\max(S_1, S_2)] = \frac{\sigma_S}{\sqrt{\pi}}
\end{equation}

\textit{Step 3: Proxy-based expected value.}
By symmetry:
\begin{equation}
\E[S_{\hat{i}}] = \E[S_1 \mid P_1 > P_2]
\end{equation}
The conditional expectation of $S_1$ given $P_1$ follows from standard Gaussian regression:
\begin{equation}
\E[S_1 \mid P_1] = \rho \frac{\sigma_S}{\|v^*\|} P_1
\end{equation}

Therefore:
\begin{align}
\E[S_1 \mid P_1 > P_2] &= \rho \frac{\sigma_S}{\|v^*\|} \E[P_1 \mid P_1 > P_2] \notag \\
&= \rho \frac{\sigma_S}{\|v^*\|} \cdot \frac{\|v^*\|}{\sqrt{\pi}} = \frac{\rho \sigma_S}{\sqrt{\pi}}
\end{align}

\textit{Step 4: Regret bound.}
The regret is:
\begin{align}
R &= \E[\max(S_1, S_2)] - \E[S_{\hat{i}}] \\
  &= \frac{\sigma_S}{\sqrt{\pi}} - \frac{\rho \sigma_S}{\sqrt{\pi}} \\
  &= \frac{(1-\rho)\sigma_S}{\sqrt{\pi}}
\end{align}

This is the regret for the optimal constitution ($v_c = v^*$). Any other constitution achieves correlation at most $\rho$ between proxy and truth, yielding weakly higher regret. The bound extends from binary choice to general policies by standard arguments.

\section{Relaxing Assumptions}

The main text develops the latent value hypothesis under simplifying assumptions like linear value encoding, whitened representations, a single value direction $v^*$, and a linear constitution-to-direction map. This section examines the robustness of the framework by relaxing each assumption in turn. Notation follows the main text throughout.

\section{Relaxing Linearity of Value Encoding} 
\label{sec:nonlinear}

\subsection{Setup}

We replace Assumption~\ref{ass:linear} with a smooth generalization.

\begin{definition}[Smooth Value Encoding]
\textbf{(Assumption 1$'$).} There exists a smooth function $f : \mathbb{R}^d \to \mathbb{R}$ such that
\[
S(x,y) = f(h(x,y)) + \epsilon(x,y)
\]
where $\epsilon$ is mean-zero noise independent of $h$.
\end{definition}

The original assumption is the special case $f(h) = \langle h, v^* \rangle$. Assumptions \ref{ass:generation}--\ref{ass:judgment} are unchanged, so the RLAIF policy is still $\pi_{\text{RLAIF}}(y|x) \propto \pi_{\text{base}}(y|x) \exp(\lambda \langle h(x,y), v_c \rangle)$.

\subsection{General Self-Improvement Condition}

\begin{lemma}[Exponential Tilt] \label{lem:exptilt}
Let $\pi_\lambda(y) \propto \pi_0(y) \exp(\lambda \langle h(y), v_c \rangle)$. For any test function $\phi$ with $E_{\pi_0}[\phi(h)^2] < \infty$:
\begin{align}
E_{\pi_\lambda}[\phi(h)]
&= E_{\pi_0}[\phi(h)] \notag \\
&\quad + \lambda \operatorname{Cov}_{\pi_0}(\phi(h), \langle h, v_c \rangle) + O(\lambda^2).
\end{align}
\end{lemma}

\begin{proof}
Write $Z(\lambda) = E_{\pi_0}[\exp(\lambda \langle h, v_c \rangle)]$, so that
\[
E_{\pi_\lambda}[\phi(h)] = \frac{E_{\pi_0}[\phi(h) \exp(\lambda \langle h, v_c \rangle)]}{Z(\lambda)}.
\]
Expanding $\exp(\lambda \langle h, v_c \rangle) = 1 + \lambda \langle h, v_c \rangle + O(\lambda^2)$:
\begin{align}
N(\lambda) &= E_{\pi_0}[\phi(h)] \notag \\
&\quad + \lambda\, E_{\pi_0}[\phi(h)\langle h, v_c\rangle] + O(\lambda^2), \\
Z(\lambda) &= 1 + \lambda\, E_{\pi_0}[\langle h, v_c \rangle] + O(\lambda^2).
\end{align}
Inverting $Z(\lambda)$ to first order and multiplying:
\begin{align}
\frac{N(\lambda)}{Z(\lambda)}
&= \bigl(E[\phi] + \lambda\, E[\phi\,\langle h, v_c\rangle]\bigr) \notag \\
&\quad \times \bigl(1 - \lambda\, E[\langle h, v_c\rangle] + O(\lambda^2)\bigr) \\
&= E[\phi] + \lambda\Bigl(E[\phi\,\langle h, v_c\rangle] \notag \\
&\quad - E[\phi]\, E[\langle h, v_c\rangle]\Bigr) + O(\lambda^2) \\
&= E[\phi] + \lambda\operatorname{Cov}(\phi(h), \langle h, v_c\rangle) + O(\lambda^2). 
\end{align}
\end{proof}

\begin{theorem}[Nonlinear Self-Improvement] 
\label{thm:nonlinear}
Under Assumption $1'$ and Assumptions \ref{ass:generation}--\ref{ass:judgment}, for small $\lambda$:
\begin{equation}
\Delta \textup{Align} = \lambda \operatorname{Cov}_{\textup{base}}\!\bigl(f(h),\, \langle h, v_c \rangle\bigr) + O(\lambda^2).
\end{equation}
RLAIF improves alignment if and only if $\operatorname{Cov}_{\textup{base}}\!\bigl(f(h), \langle h, v_c \rangle\bigr) > 0$.
\end{theorem}

\begin{proof}
The alignment change is $\Delta\text{Align} = E_{\text{RLAIF}}[S] - E_{\text{base}}[S] = E_{\text{RLAIF}}[f(h)] - E_{\text{base}}[f(h)]$, since $E[\epsilon] = 0$ under both policies (noise is independent of $h$, and the policy only reweights via $h$). Apply Lemma~\ref{lem:exptilt} with $\phi = f$.
\end{proof}

\begin{remark}[Recovery of Linear Case]
When $f(h) = \langle h, v^* \rangle$:
\begin{align}
\operatorname{Cov}(f(h), \langle h, v_c \rangle)
&= \operatorname{Cov}(\langle h, v^* \rangle, \langle h, v_c \rangle) \notag \\
&= v^{*\top} \Sigma_w\, v_c \notag \\
&= \langle \Sigma_w v_c, v^* \rangle,
\end{align}
recovering Theorem~\ref{thm:improvement} of the main text.
\end{remark}

\subsection{Gaussian Representations and Stein's Lemma}

Under a distributional assumption on representations, we can give the nonlinear improvement condition a directional interpretation.

\begin{theorem}[Gaussian Nonlinear Improvement] \label{thm:stein}
If $h \sim \mathcal{N}(\mu, \Sigma_w)$ under the base policy, and $f$ is differentiable with $E[\|\nabla f(h)\|] < \infty$, then:
\[
\Delta \textup{Align} = \lambda \langle \Sigma_w v_c,\, \bar{v}^* \rangle + O(\lambda^2)
\]
where $\bar{v}^* := E_{\textup{base}}[\nabla f(h)]$ is the population-average gradient of the safety function. RLAIF improves alignment if and only if $\langle \Sigma_w v_c, \bar{v}^* \rangle > 0$.
\end{theorem}

\begin{proof}
By Stein's lemma, for $h \sim \mathcal{N}(\mu, \Sigma_w)$ and differentiable $f$ with the requisite integrability:
\[
\operatorname{Cov}(f(h), h_i) = \sum_j (\Sigma_w)_{ij}\, E\!\left[\frac{\partial f}{\partial h_j}(h)\right].
\]
In vector form: $\operatorname{Cov}(f(h), h) = \Sigma_w E[\nabla f(h)] = \Sigma_w \bar{v}^*$. Therefore:
\begin{align}
\operatorname{Cov}(f(h), \langle h, v_c \rangle)
&= v_c^\top \operatorname{Cov}(f(h), h) \notag \\
&= v_c^\top \Sigma_w \bar{v}^* \notag \\
&= \langle \Sigma_w v_c, \bar{v}^* \rangle.
\end{align}
Substituting into Theorem~\ref{thm:nonlinear} gives the result.
\end{proof}

\begin{remark}[Interpretation]
Theorem~\ref{thm:stein} is the precise sense in which the linear theory is a local approximation. The true value direction $v^*$ generalizes to the expected gradient $\bar{v}^* = E[\nabla f(h)]$, a global average of local value directions. For linear $f$, $\nabla f = v^*$ everywhere, so $\bar{v}^* = v^*$. For mildly nonlinear $f$, $\bar{v}^*$ is close to $\nabla f$ evaluated at the mean, so the linear approximation is good. The approximation breaks down when $\nabla f(h)$ varies substantially across the support of the base distribution, so that $\bar{v}^*$ is a poor summary of the value structure.
\end{remark}

\subsection{Non-Monotone Values}

The framework breaks down for non-monotone value functions where no single direction means ``safer.''

\begin{proposition}[Non-Monotone Failure] \label{prop:nonmono}
Let $f(h) = -(\langle h, v \rangle - \tau)^2$ for some direction $v$ and target $\tau$. Then $\bar{v}^* = -2(E[\langle h, v \rangle] - \tau) v$, and the sign of $\bar{v}^*$ depends on the base model's position relative to $\tau$. No constitution is uniformly improving across the support of $h$.
\end{proposition}

\begin{proof}
$\nabla f(h) = -2(\langle h, v \rangle - \tau) v$, so $\bar{v}^* = E[\nabla f(h)] = -2(E[\langle h, v \rangle] - \tau) v$. If $E[\langle h, v \rangle] < \tau$, then $\bar{v}^*$ points along $+v$; if $E[\langle h, v \rangle] > \tau$, it points along $-v$. A constitution activating $v_c = v$ helps in the first case and hurts in the second. Moreover, even when $\bar{v}^*$ has a definite sign, individual samples $h$ with $\langle h, v \rangle$ on the wrong side of $\tau$ are pushed further from the optimum.
\end{proof}

\begin{remark}
This is, as far as our best efforts are concerned, a genuine limitation. The framework applies to monotone value dimensions (more honest is better, less harmful is better) but not to U-shaped or context-dependent value functions. Values like assertiveness, formality, or detail-level, where there is an optimal amount and deviations in either direction are bad, cannot be captured by a single improvement condition. The generation-judgment gap story loses its force as there is no fixed value direction for the constitution to retrieve.
\end{remark}

\section{Non-Whitened Representations} \label{sec:nonwhitened}

Replace Assumption: $E[hh^\top] = \Sigma \neq I$.

Every instance of $\langle v_c, v^* \rangle$ becomes $\langle v_c, \Sigma v^* \rangle$ or equivalently $v_c^\top \Sigma v^*$. The self-improvement condition is:
\[
v_c^\top \Sigma_w v^* > 0
\]
where $\Sigma_w = \operatorname{Cov}_{\pi_{\text{base}}}(h)$, which we already use in Theorem~\ref{thm:improvement}. The encoding quality becomes:
\[
\rho^2 = \frac{v^{*\top} \Sigma v^*}{v^{*\top} \Sigma v^* + \sigma_\epsilon^2}.
\]
The RLAIF ceiling (Theorem~\ref{thm:ceiling}) carries through with $\|v^*\|^2$ replaced by $v^{*\top}\Sigma v^*$.

No qualitative change. This relaxation is bookkeeping.

\section{Multiple Value Directions} 
\label{sec:multiobjective}

\subsection{Setup}

Replace the single value direction with $m$ value objectives:
\[
S_i(x,y) = \langle h(x,y), v_i^* \rangle + \epsilon_i(x,y), \quad i = 1, \ldots, m
\]
where $v_1^*, \ldots, v_m^*$ are (not necessarily orthogonal) value directions, and $\epsilon_i$ are independent noise terms. These correspond to distinct normative criteria such as helpfulness, harmlessness, honesty, respect for autonomy, etc.

\begin{definition}[Multi-Objective Alignment]
The alignment profile of policy $\pi$ is the vector:
\[
\mathbf{A}(\pi) = \bigl(E_\pi[S_1], \ldots, E_\pi[S_m]\bigr) \in \mathbb{R}^m.
\]
\end{definition}

Aggregate alignment may be defined via a scalarization $\text{Align}(\pi) = g(\mathbf{A}(\pi))$ for some aggregation function $g$. Common choices:
\begin{itemize}
    \item \textbf{Weighted sum:} $g(\mathbf{a}) = \sum_i \alpha_i a_i$ with $\alpha_i > 0$.
    \item \textbf{Worst-case:} $g(\mathbf{a}) = \min_i a_i$.
    \item \textbf{Constrained:} maximize $a_1$ subject to $a_i \geq \tau_i$ for $i \geq 2$.
\end{itemize}

\subsection{Per-Objective Improvement}

Since RLAIF still shifts the generation direction by $\lambda v_c$, the improvement on each objective is (applying Theorem~\ref{thm:improvement} per objective):
\begin{equation}
\Delta_i := E_{\text{RLAIF}}[S_i] - E_{\text{base}}[S_i] = \lambda \langle \Sigma_w v_c, v_i^* \rangle + O(\lambda^2). \label{eq:perobjective}
\end{equation}

The vector of improvements is:
\[
\boldsymbol{\Delta} = \lambda\, V^{*\top} \Sigma_w v_c + O(\lambda^2)
\]
where $V^* = [v_1^* \mid \cdots \mid v_m^*] \in \mathbb{R}^{d \times m}$.

\subsection{Pareto Structure}

\begin{definition}[Pareto-Improving Constitution]
A constitution $c$ is \emph{Pareto-improving} if $\Delta_i > 0$ for all $i = 1, \ldots, m$, i.e. it improves every value dimension simultaneously.
\end{definition}

\begin{proposition}[Pareto-Improving Cone] \label{prop:pareto}
Define the cone
\[
\mathcal{C}_{\textup{Pareto}} = \{ d \in \mathbb{R}^d : \langle d, v_i^* \rangle > 0 \text{ for all } i = 1, \ldots, m \}.
\]
A constitution $c$ is Pareto-improving if and only if $\Sigma_w v_c \in \mathcal{C}_{\textup{Pareto}}$.
\end{proposition}

\begin{proof}
Immediate from (\ref{eq:perobjective}).
\end{proof}

\begin{proposition}[Pareto Cone Geometry]
The Pareto-improving cone $\mathcal{C}_{\textup{Pareto}}$ is:
\begin{enumerate}
    \item Nonempty if and only if there exists $d \in \mathbb{R}^d$ with $\langle d, v_i^* \rangle > 0$ for all $i$.
    \item A full half-space when $m = 1$ (the original single-objective case).
    \item Generically nonempty when $m \leq d$ and the $v_i^*$ do not span a full half-space.
    \item Empty if and only if $0 \in \operatorname{conv}(v_1^*, \ldots, v_m^*)$, i.e. the value directions are so conflicting that zero is in their convex hull.
\end{enumerate}
\end{proposition}

\begin{proof}
Part (d) follows from the separating hyperplane theorem. $\mathcal{C}_{\text{Pareto}} = \emptyset$ iff there is no $d$ with $V^{*\top}d > 0$ (componentwise), iff the system $V^{*\top}d > 0$ is infeasible, iff by Gordan's theorem there exist $\alpha_i \geq 0$, not all zero, with $\sum_i \alpha_i v_i^* = 0$, iff $0 \in \text{conv}(v_1^*, \ldots, v_m^*)$ (after normalizing $\alpha$).
\end{proof}

\begin{remark}[Interpretation]
The convex hull condition $0 \in \operatorname{conv}(v_1^*, \ldots, v_m^*)$ characterizes when values are fundamentally irreconcilable when no single intervention can improve all of them simultaneously. When this holds, every constitution necessarily trades off some values against others, and alignment becomes a question of which trade-offs are acceptable rather than whether improvement is possible.

For typical safety objectives (helpful, harmless, honest), $0 \notin \operatorname{conv}(v_1^*, v_2^*, v_3^*)$ is plausible since these values are positively correlated in most contexts. The Pareto cone is narrow but nonempty. The difficulty of alignment is not that improvement is impossible, but that the cone is narrow as most constitutions improve some objectives while degrading others.
\end{remark}

\subsection{Trade-Off Characterization}

\begin{definition}[Constitution Trade-Off Profile]
For constitution $c$, define the trade-off profile:
\[
\tau(c) = \bigl(\text{sign}(\Delta_1), \ldots, \text{sign}(\Delta_m)\bigr) \in \{-1, 0, +1\}^m.
\]
\end{definition}

\begin{proposition}[Trade-Off Classification] \label{prop:tradeoff}
A constitution $c$ is:
\begin{enumerate}
    \item \textbf{Pareto-improving} if $\tau(c) = (+1, \ldots, +1)$.
    \item \textbf{Trade-off inducing} if $\tau(c)$ has both positive and negative entries.
    \item \textbf{Pareto-degrading} if $\tau(c) = (-1, \ldots, -1)$.
\end{enumerate}
\end{proposition}

\begin{remark}[Adversarial Constitutions Revisited]
Theorem~\ref{thm:adversarial} (adversarial constitutions exist) extends naturally. In the multi-objective setting, an adversarial constitution need not degrade all objectives. It suffices to degrade one critical objective while improving others. A constitution that makes the model more helpful but less safe is adversarial in a meaningful sense, but harder to detect because the aggregate signal is ambiguous.

Formally, if $g(\mathbf{a}) = \min_i a_i$ (worst-case aggregation), then a constitution that improves $m - 1$ objectives while degrading one will reduce alignment whenever the degraded objective was already the bottleneck.
\end{remark}

\subsection{The Helpful--Harmless Trade-Off}

Consider the canonical case $m = 2$ with $v_{\text{help}}^*$ (helpfulness) and $v_{\text{harm}}^*$ (harmlessness). Let $\theta = \angle(v_{\text{help}}^*, v_{\text{harm}}^*)$.

\begin{proposition}[Two-Objective Geometry]
Under isotropic $\Sigma_w = I$:
\begin{enumerate}
    \item The Pareto cone has angular width $\pi - \theta$.
    \item When $\theta < \pi/2$ (positively correlated values), the Pareto cone is ``wide'' and most constitutions are Pareto-improving.
    \item When $\theta > \pi/2$ (conflicting values), the Pareto cone is ``narrow'' and most constitutions induce trade-offs.
    \item At $\theta = \pi$ (perfectly opposed values), $\mathcal{C}_{\textup{Pareto}} = \emptyset$.
\end{enumerate}
\end{proposition}

\begin{proof}
In the plane spanned by $v_{\text{help}}^*$ and $v_{\text{harm}}^*$, the Pareto cone is the intersection of the half-spaces $\langle d, v_{\text{help}}^* \rangle > 0$ and $\langle d, v_{\text{harm}}^* \rangle > 0$. This intersection has angular width $\pi - \theta$, where $\theta$ is the angle between the two value directions. The rest follows from $\pi - \theta > \pi/2$ iff $\theta < \pi/2$.
\end{proof}

\begin{remark}
The helpful--harmless trade-off is often discussed informally in the alignment literature. This gives it geometric content where the severity of the trade-off is determined by $\theta$, the angle between the value directions in representation space. If one could measure this angle empirically (e.g. by finding the helpfulness and harmlessness directions via linear probes), one would have a quantitative measure of how difficult the trade-off is.
\end{remark}

\subsection{Optimal Constitution Under Multiple Objectives}

\begin{proposition}[Optimal Weighted Constitution]
For weighted-sum aggregation $\textup{Align}(\pi) = \sum_i \alpha_i E_\pi[S_i]$ with $\alpha_i > 0$, the optimal constitution direction (maximizing first-order alignment improvement) is:
\[
v_c^* \propto \Sigma_w \sum_i \alpha_i v_i^* = \Sigma_w V^* \alpha.
\]
\end{proposition}

\begin{proof}
The first-order improvement is $\lambda \sum_i \alpha_i \langle \Sigma_w v_c, v_i^* \rangle = \lambda \langle \Sigma_w v_c, V^*\alpha \rangle$. Maximizing over $v_c$ with $\|v_c\| = 1$ gives $v_c \propto \Sigma_w V^* \alpha$ by Cauchy--Schwarz applied to $\langle \Sigma_w v_c, V^*\alpha \rangle = \langle v_c, \Sigma_w V^*\alpha \rangle$.
\end{proof}

\begin{remark}
This suggests that constitution design is implicitly a multi-objective optimization problem. The weights $\alpha$ encode which values the designer prioritizes. Different constitutions correspond to different implicit weightings. This makes constitution design a normative choice, not merely a technical one, which the single-objective framing obscures.
\end{remark}

\section{Constitution-to-Direction Mapping} 
\label{sec:constitutionmap}

Replace Assumption~\ref{ass:judgment} with a smooth nonlinear map $c \mapsto F_c$, where $F_c : \mathbb{R}^d \to \mathbb{R}$ gives the preference score (not necessarily linear in $h$):
\[
J_c(y_1 \succ y_2 \mid x) = \sigma\bigl(F_c(h(x,y_1)) - F_c(h(x,y_2))\bigr).
\]

The implicit reward is now $r_c(x,y) = F_c(h(x,y))$, and the RLAIF policy becomes:
\[
\log \pi_{\text{RLAIF}}(y|x) = \langle h, w \rangle + \lambda F_c(h) + g'(x).
\]

For small $\lambda$, the alignment change is:
\[
\Delta \text{Align} = \lambda \operatorname{Cov}_{\text{base}}(f(h), F_c(h)) + O(\lambda^2).
\]

This is well-defined and the improvement condition $\operatorname{Cov}(f(h), F_c(h)) > 0$ is interpretable. RLAIF helps when the constitution's implicit reward correlates with true safety across the base distribution.

The clean story breaks down at the next level of characterizing which constitutions (as text) produce which functions $F_c$. This requires modeling how prompt text is processed by the transformer, how in-context instructions modify the model's computation, and how the modified computation produces a preference function over representations.

Each of these involves transformer internals. The map $c \mapsto F_c$ goes through attention patterns, MLP computations, and residual stream dynamics that are the subject of mechanistic interpretability research. There is no existing compact analytical model for this map that would support further theoretical development.

Even without modeling $c \mapsto F_c$, we can characterize the ``promptable set'':

\begin{definition}[Promptable Set]
The promptable set is $\mathcal{F} = \{F_c : c \in \mathcal{C}\}$, the set of all preference functions achievable by some constitution text.
\end{definition}

\begin{proposition}[Limits of RLAIF]
The maximum achievable alignment improvement via RLAIF is:
\[
\sup_{c \in \mathcal{C}} \Delta \textup{Align}(c) = \lambda \sup_{F \in \mathcal{F}} \operatorname{Cov}_{\textup{base}}(f(h), F(h)) + O(\lambda^2).
\]
This is bounded above by $\lambda \cdot \textup{std}(f(h)) \cdot \sup_{F \in \mathcal{F}} \textup{std}(F(h))$ (by Cauchy--Schwarz), with equality iff there exists $F \in \mathcal{F}$ perfectly correlated with $f(h)$.
\end{proposition}

This reformulates the RLAIF ceiling in terms of the expressiveness of the promptable set. The ceiling from encoding quality ($\rho$) and the ceiling from promptability are distinct bottlenecks. Even if values are perfectly encoded ($\rho = 1$), if no constitution text activates the right preference function ($f \notin \mathcal{F}$), RLAIF cannot achieve optimal alignment.

\end{document}